\documentclass[runningheads]{llncs}
\usepackage{graphicx}
\usepackage[linesnumbered,ruled,vlined]{algorithm2e}
\usepackage{mathrsfs}

\usepackage{amsmath}

\usepackage[show]{chato-notes}

\usepackage{tikz}
\usepackage{gincltex}

\SetKwInput{KwInput}{Input}                
\SetKwInput{KwOutput}{Output}              
\newcommand{\htheta}{\hat{\theta}}
\newcommand{\ruleset}{M}
\newcommand{\assign}{\leftarrow}
\newcommand{\turs}{\textsc{Turs} }
\renewcommand\vec{\boldsymbol}
\begin{document}
%
\title{Truly Unordered Probabilistic Rule Sets\\for Multi-class Classification}

\author{Lincen Yang\orcidID{0000-0003-1936-2784} \and
Matthijs van Leeuwen\orcidID{0000-0002-0510-3549}}
\authorrunning{Yang, L and van Leeuwen, M.}

\institute{LIACS, Leiden University, Leiden, The Netherlands \\
\email{{l.yang, m.van.leeuwen}@liacs.leidenuniv.nl}}

%
%
\maketitle              
\begin{abstract}
Rule set learning has long been studied and has recently been frequently revisited due to the need for interpretable models. Still, existing methods have several shortcomings: 1) most recent methods require a binary feature matrix as input, while learning rules directly from numeric variables is understudied; 2) existing methods impose orders among rules, either explicitly or implicitly, which harms interpretability; and 3) currently no method exists for learning probabilistic rule sets for multi-class target variables (there is only one for probabilistic rule lists). 

We propose \textsc{Turs}, for Truly Unordered Rule Sets, which addresses these shortcomings. We first formalize the problem of learning truly unordered rule sets. To resolve conflicts caused by overlapping rules, i.e., instances covered by multiple rules, we propose a novel approach that exploits the probabilistic properties of our rule sets. We next develop a two-phase heuristic algorithm that learns rule sets by carefully growing rules. An important innovation is that we use a surrogate score to take the global potential of the rule set into account when learning a local rule.

Finally, we empirically demonstrate that, compared to non-probabilistic and (explicitly or implicitly) ordered state-of-the-art methods, our method learns rule sets that not only have better interpretability but also better predictive performance. 


\end{abstract}
\section{Introduction} \label{sec:intro}
When using predictive models in sensitive real-world scenarios, such as in health care, analysts seek for intelligible and reliable explanations for predictions. Classification rules have considerable advantages here, as they are directly readable by humans. While rules all seem alike, however, some are more interpretable than others. The reason lies in the subtle differences of how rules form a model. Specifically, rules can form an unordered \emph{rule set}, or an explicitly ordered \emph{rule list}; further, they can be categorized as probabilistic or non-probabilistic. 

In practice, probabilistic rules should be preferred because they provide information about the uncertainty of the predicted outcomes, and thus are useful when a human is responsible to make the final decision, as the expected ``utility" can be calculated. 
 Meanwhile, unordered rule sets should also be preferred, as they have better properties regarding interpretability than ordered rule lists. 
 While no agreement has been reached on the precise definition of interpretability of machine learning models 
 \cite{murdoch2019interpretable,molnar2020interpretable}, we specifically treat interpretability with domain experts in mind. From this perspective, a model's interpretability intuitively depends on two aspects: the degree of difficulty for a human to comprehend the model itself, and to understand a single prediction. Unordered probabilistic rule sets are favorable with respect to both aspects, for the following reasons. First, comprehending ordered rule lists requires comprehending not only each individual rule, but also the relationship among the rules, while comprehending unordered rule sets requires only the former. Second, the explanation for a single prediction of an ordered rule list must contain the rule that the instance satisfies, together with all of its preceding rules, which becomes incomprehensible when the number of preceding rules is large. 

Further, crucially, existing methods for rule set learning claim to learn unordered rule sets, but most of them are not truly unordered. The problem is caused by \emph{overlap}, i.e., a single instance satisfying multiple rules. Ad-hoc schemes are widely used to resolve prediction conflicts caused by overlaps, typically by ranking the involved rules with certain criteria and always selecting the highest ranked rule \cite{zhang2020diverseRuleSets,lakkaraju2016interpretable} (e.g., the most accurate one). This, however, imposes implicit orders among rules, making them entangled instead of truly unordered. 

This can badly harm interpretability: to explain a single prediction for an instance, it is now insufficient to only provide the rules the instance satisfies, because other higher-ranked rules that the instance does \emph{not} satisfy are also part of the explanation. For instance, imagine a patient is predicted to have \emph{Flu} because they have \emph{Fever}. If the model also contains the higher-ranked rule \emph{``Blood in stool $\rightarrow$ Dysentery"}, the explanation should include the fact that \emph{``Blood in stool"} is not true, because otherwise the prediction would change to \emph{Dysentery}. If the model contains many rules, it becomes impractical to have to go over all higher-ranked rules for each prediction. 

Learning truly unordered probabilistic rule sets is a very challenging problem though. Classical rule set learning methods usually adopt a separate-and-conquer strategy, often sequential covering: they iteratively find the next rule and remove instances satisfying this rule. This includes 1) binary classifiers that learn rules only for the ``positive" class \cite{furnkranz2012foundations}, and 2) its extension to multi-class targets by the one-versus-rest paradigm, i.e., learning rules for each class one by one \cite{cohen1995ripper,clark1991cn2Improve}. Importantly, by iteratively removing instances the \emph{probabilistic predictive conflicts} caused by overlaps, i.e., rules having different probability estimates for the target, are ignored. Recently proposed rule learning methods go beyond separate-and-conquer by leveraging discrete optimization techniques \cite{zhang2020diverseRuleSets,wang2017bayesian,yang2021learning,lakkaraju2016interpretable,dash2018boolean}, but this comes at the cost of requiring a binary feature matrix as input. Moreover, these methods are neither probabilistic nor truly unordered, as they still use ad-hoc schemes to resolve predictive conflicts caused by overlaps. 

\emph{Approach and contributions.}
To tackle these challenges and learn truly unordered probabilistic rules, we first formalize rule sets as probabilistic models. We adopt a probabilistic model selection approach for rule set learning, for which we design a criterion based on the minimum description length (MDL) principle \cite{grunwald2019minimum}. Second, we propose a novel surrogate score based on decision trees that we use to evaluate the potential of incomplete rule sets. Third, we are the first to design a rule learning algorithm that deals with probabilistic conflicts caused by overlaps already during the rule learning process. We point out that rules that have been added to the rule set may become obstacles for new rules, and hence carefully design a two-phase heuristic algorithm, for which we adopt diverse beam search \cite{matthijs2012diverse}. 
Last, we benchmark our method, named \textsc{Turs}, for Truly Unordered Rule Sets, against a wide range of methods. We show that the rule sets learned by \textsc{Turs}, apart from being probabilistic and truly unordered, have better predictive performance than existing rule list and rule set methods.

\section{Related Work} \label{sec:related}

\noindent
\emph{Rule lists.}
Rules in a rule list are connected by \textsc{if-then-else} statements. Existing methods include CBA~\cite{liu1998CBA}, ordered CN2~\cite{clark1989cn2}, PART~\cite{frank1998generating}, and the recently proposed CLASSY~\cite{proencca2020interpretable} and Bayesian rule list~\cite{yang2017scalable-bayesian-rule-list}. We argue that rule lists are more difficult to interpret than rule sets because of their explicit orders. 

\smallskip \noindent
\emph{One-versus-rest learning.}
This category focuses on only learning rules for a single class label, i.e., the ``positive" class, which is already sufficient for binary classification \cite{wang2017bayesian,dash2018boolean,yang2021learning}. For multi-class classification, two approaches exist. The first, taken by RIPPER~\cite{cohen1995ripper} and C4.5~\cite{quinlan2014c4}, is to learn each class in a certain order. After all rules for a single class have been learned, all covered instances are removed (or those with this class label). The resulting model is essentially an ordered list of rule sets, and hence is more difficult to interpret than rule set.

The second approach does not impose an order among the classes; instead, it learns a set of rules for each class against all other classes. The most well-known are unordered-CN2 and FURIA~\cite{clark1991cn2Improve,huhn2009furia}. 
FURIA avoids dealing with conflicts of overlaps by using all rules for predicting unseen instances; as a result, it cannot provide a single rule to explain its prediction. Unordered-CN2, on the other hand, handles overlaps by ``combining" all overlapping rules into a ``hypothetical" rule, which sums up all instances in all overlapping rules and hence ignoring probabilistic conflicts for constructing rules. In Section~\ref{sec:exp}, we show that our method learns smaller rule sets with better predictive performance than unordered-CN2.

\smallskip \noindent
\emph{Multi-class rule sets.}
Very few methods exist for directly learning rules for multi-class targets, which is algorithmically more challenging than the one-versus-rest paradigm, as the separate-and-conquer strategy is not applicable. 
To the best of our knowledge, the only existing methods are IDS~\cite{lakkaraju2016interpretable} and DRS~\cite{zhang2020diverseRuleSets}. Both are neither probabilistic nor truly unordered. To handle conflicts of overlaps, IDS follows the rule with the highest F1-score, and DRS uses the most accurate rule. 

Last, different but related approaches include 1) decision tree based methods such as CART~\cite{breiman1984classification}, which produce rules that are forced to share many ``attributes" and hence are longer than necessary, as we will empirically demonstrate in Section~\ref{sec:exp}, and 2) a Bayesian rule mining~\cite{gay2012bayesian} method, which adopts naive bayes  with the mined rules for prediction, and hence does not produce a rule set model in the end. The `lazy learning' approach for rule-based models can also avoid the conflicts of overlaps~\cite{veloso2006lazy}, but no global rule set model describing the whole dataset is constructed in this case.

\section{Rule Sets as Probabilistic Models}\label{section:rulesetModel}

We first formalize individual rules as \emph{local} probabilistic models, and then define rule sets as \emph{global} probabilistic models. The key challenge lies in how to define $P(Y=y|X=x)$ for an instance $(x,y)$ that is covered by multiple rules. 

\subsection{Probabilistic Rules}
Denote the input random variables by $X = (X_1, \ldots, X_d)$, where each $X_i$ is a one-dimensional random variable representing one dimension of $X$, and denote the categorical target variable by $Y\in\mathscr{Y}$. Further, denote the dataset from which the rule set can be induced as $D = \{(x_i, y_i)\}_{i \in [n]}$, or $(x^n, y^n)$ for short. Each $(x_i,y_i)$ is an instance. Then, a probabilistic rule $S$ is written as
\begin{equation}
(X_1 \in R_1 \land X_2 \in R_2 \land \ldots) \rightarrow P_S(Y),
\end{equation}
where each $X_i \in R_i$ is called a \emph{literal} of the \emph{condition} of the rule. Specifically, each $R_i$ is an interval (for a quantitative variable) or a set of categorical levels (for a categorical variable). 

A probabilistic rule of this form describes a subset $S$ of the full sample space of $X$, such that for any $x \in S$, the conditional distribution $P(Y | X=x)$ is approximated by the probability distribution of $Y$ conditioned on the event $\{X \in S\}$, denoted as $P(Y | X \in S)$. Since in classification $Y$ is a discrete variable, we can parametrize $P(Y|X\in S)$ by a parameter vector $\vec{\beta}$, in which the $j$th element $\beta_j$ represents $P(Y=j|X\in S)$, for all $j \in \mathscr{Y}$. We therefore denote $P(Y | X \in S)$ as $P_{S, \vec{\beta}}(Y)$, or $P_S(Y)$ for short. To estimate $\vec{\beta}$ from data, we adopt the maximum likelihood estimator, denoted as $P_{S, \hat{\vec{\beta}}}(Y)$, or $\hat{P}_S(Y)$ for short.

Further, if an instance $(x,y)$ satisfies the condition of rule $S$, we say that $(x,y)$ is \emph{covered} by $S$. Reversely, the \emph{cover} of $S$ denotes the instances it covers. When clear from the context, we use $S$ to both represent the rule itself and/or its cover, and define the number of covered instances $|S|$ as its \emph{coverage}.

\subsection{Truly Unordered Rule Sets as Probabilistic Models} \label{subsec:probRuleSet}

While a rule set is simply a set of rules, the challenge lies in how to define rule sets as probabilistic models while keeping the rules truly unordered. That is, how do we define $P(Y|X=x)$ given a rule set $\ruleset$, i.e., a model, and its parameters? We first explain how to do this for a single instance of the training data, using a simplified setting where at most two rules cover the instance. We then discuss---potentially unseen---test instances and extend to more than two rules covering an instance. Finally, we define a rule set as a probabilistic model.


\medskip \noindent
\textbf{Class probabilities for a single training instance.} 
Given a rule set $\ruleset$ with $K$ individual rules, denoted $\{S_i\}_{i \in [K]}$, any instance $(x,y)$ falls into one of four cases: 1) exactly one rule covers $x$; 2) at least two rules cover $x$ and no rule's cover is the subset of another rule's cover (\emph{multiple non-nested}); 3) at least two rules cover $x$ and one rule's cover is the subset of another rule's cover (\emph{multiple nested}); and 4) no rule in $M$ covers $x$.

To simplify the notation, we here consider at most two rules covering an instance---we later describe how we can trivially extend to more than two rules.

\smallskip \noindent
\emph{Covered by one rule.} 
When exactly one rule $S \in \ruleset$ covers $x$, we use $P_S(Y)$ to ``approximate" the conditional probability $P(Y|X=x)$. To estimate $P_S(Y)$ from data, we adopt the maximum likelihood (ML) estimator $\hat{P}_S(Y)$, i.e.,
\begin{equation}
\hat{P}_S(Y = j) = \frac{|\{(x,y): x \in S, y = j\}|}{|S|}, \forall j \in \mathscr{Y}. 
\end{equation}
Note that we do not exclude instances in $S$ that are also covered by other rules (i.e., in overlaps) for estimating $P_S(Y)$. Hence, the probability estimation for each rule is independent of other rules; as a result, each rule is \emph{self-standing}, which forms the foundation of a truly unordered rule set. 

\smallskip \noindent
\emph{Covered by two non-nested rules.}
Next, we consider the case when $x$ is covered by 
$S_i$ and $S_j$, and neither $S_i \subseteq S_j$ nor $S_j \subseteq S_i$, i.e., the rules are non-nested.  


When an instance is covered by two non-nested, partially overlapping rules, we interpret this as probabilistic \emph{uncertainty}: we cannot tell whether the instance belongs to one rule or the other, and therefore approximate its conditional probability by the \emph{union} of the two rules. That is, in this case we approximate $P(Y|X=x)$ by $P(Y|X \in S_i \cup S_j)$, and we estimate this with its ML estimator $\hat{P}(Y|X \in S_i \cup S_j)$, using all instances in $S_i \cup S_j$. 

This approach is particularly useful when the estimator of $P(Y|X \in S_i \cap S_j)$, i.e., conditioned on the event $\{X \in S_i \cap S_j\}$, is indistinguishable from $\hat{P}(Y|X \in S_i)$ and $\hat{P}(Y|X \in S_j)$. Intuitively, this can be caused by two reasons: 1) $S_i \cap S_j$ consists of very few instances, so the variance of the estimator for $P(Y|X \in S_i \cap S_j)$ is large; 2) $P(Y|X \in S_i \cap S_j)$ is just very similar to $P(Y|X \in S_i)$ and $P(Y|X \in S_i)$, which makes it undesirable to create a separate rule for $S_i \cap S_j$. Our model selection approach, explained in Section~\ref{sec:model_selection}, will ensure that a rule set with non-nested rules has high goodness-of-fit only if this `uncertainty' is indeed the case.

\smallskip \noindent
\emph{Covered by two nested rules.}
When $x$ is covered by both $S_i$ and $S_j$, and $S_i$ is a subset of $S_j$, i.e., $x \in S_i \subseteq S_j$, the rules are nested\footnote{Note that ``nestedness'' is based on the rules' covers rather than on their conditions. For instance, if $S_i$ is $X_1 <= 1$ and $S_j$ is $X_2 <= 1$, $S_i$ and $S_j$ could still be nested.}.
In this case, we approximate $P(Y|X=x)$ by $P(Y | X \in S_i)$ and interpret $S_i$ as an \emph{exception} of $S_j$. 
Having such nested rules to model such exceptions is intuitively desirable, as it allows to have general rules covering large parts of the data while being able to model smaller, deviating parts. 
In order to preserve the self-standing property of individual rules, for $x \in S_j \setminus S_i$ we still use $P(Y|X \in S_j)$ rather than $P(Y | X \in S_j \setminus S_i)$. Although this might seem counter-intuitive at first glance, using $P(Y | X \in S_j \setminus S_i)$ would implicitly impose an order between $S_j$ and $S_i$, or---equivalently---implicitly change $S_j$ to another rule that only covers instances in $S_j \land \neg S_i$. 
 
\smallskip \noindent
\emph{Not covered by any rule.}
When no rule in $M$ covers $x$, we say that $x$ belongs to the so-called ``else rule'' that is part of every rule set and equivalent to $x \notin \bigcup_{i} S_i$. Thus, we approximate $P(Y|X=x)$ by $P(Y | X \notin \bigcup_{i} S_i)$. We denote the else rule by $S_0$ and write $S_0 \in \ruleset$ for the else rule in $\ruleset$. Observe that the else rule is the only rule in every rule set that depends on the other rules and is therefore not self-standing; however, it will also have no overlap with other rules by definition.

\medskip \noindent
\textbf{Predicting for a new instance.} \label{subsubsec:new_data}
When an unseen instance $x'$ comes in, we predict $P(Y|X=x')$ depending on which of the aforementioned four cases it satisfies. An important question is whether we always need access to the training data, i.e., whether the probability estimates we obtain from the training data points are sufficient for predicting $P(Y|X=x')$.
Specifically, if $x'$ is covered by non-nested $S_i$ and $S_j$, $P(Y|X=x')$ is predicted as $\hat{P}(Y|X \in S_i \cup S_j)$. However, if there are no training data points covered both by $S_i$ and $S_j$, then we would not obtain $\hat{P}(Y|X \in S_i \cup S_j)$ in the training phase. Nevertheless, in this case we have  $|S_i \cup S_j| = |S_i| + |S_j|$, and hence
\begin{equation}
    \hat{P}(Y|X \in S_i \cup S_j) = \frac{|S_i| \hat{P}(Y|X\in S_i) + |S_j| \hat{P}(Y|X\in S_j)}{|S_i| + |S_j|}.
\end{equation}

Thus, if $x'$ is covered by one rule, two nested rules, or no rule in $M$, the corresponding probability estimates are already obtained during training. Thus, we conclude that access to the training data is not necessary for prediction.

\medskip \noindent
\textbf{Extension to overlaps of multiple rules.}
Whenever an instance $x$ is covered by multiple rules, denoted $J = \{S_i, S_j, S_k, ...\}$, three cases may happen. The first case is all rules in $J$ are nested. Without loss of generality, assume that $S_i \subseteq S_j \subseteq S_k \subseteq ...$; then, following the rationale for case of two nested rules, $P(Y|X=x)$ should be approximated by $P_{S_i}(Y)$. 
Therefore, when $x$ is covered by multiple nested rules, only the ``smallest" rule matters and we can act as if $x$ is only covered by that single rule. 

The second case is that all rules in $J$ are non-nested with each other. Following the solution for modeling two non-nested rules, we use $P(Y|X \in \bigcup_{S\in J} S)$. 

The third case is a mix of the previous two cases, i.e., rules in $J$ are partially nested. In this case, we iteratively go over all $S \in J$: if there exists an $S' \in J$ satisfying $S' \subseteq S$ we remove $S$ from $J$, and continue iterating until no nested overlap in $J$ remains. If one single rule is left, we act as if $x$ is covered by that single rule; otherwise, we follow the paradigm of modeling the non-nested overlaps with the rules left in $J$.


\medskip \noindent
\textbf{Probabilistic rule sets.}
We can now build upon the previous to define rule sets as probabilistic models. 
Formally, the probabilistic model corresponding to a rule set $\ruleset$ is a family of probability distributions, denoted $P_{\ruleset, \theta}(Y|X)$ and parametrized by $\theta$. Specifically, $\theta$ is a parameter vector representing all necessary probabilities of $Y$ conditioned on events $\{X \in G\}$, where $G$ is either a single rule or the union of multiple rules. $\theta$ is estimated from data by estimating each $P(Y|X \in G)$ by its maximum likelihood estimator. The resulting estimated vector is denoted as $\htheta$ and contains $\hat{P}(Y|X \in G)$ for all ${G \in \mathscr{G}}$, where $\mathscr{G}$ consists of all individual rules and the unions of overlapping rules in $M$. 

Finally, we assume the dataset $D = (x^n, y^n)$ to be i.i.d. Specifically, let us define $(x,y) \vdash G$ for the following two cases: 1) when $G$ is a single rule (including the else rule), then $(x, y) \vdash G \iff x \in G$; and 2) when $G$ is a union of multiple rules, e.g., $G = \bigcup S_i$,  then $(x, y) \vdash G \iff x \in \bigcap S_i$. We then have
\begin{equation}
    P_{\ruleset, \theta}(y^n|x^n) = \prod_{G \in \mathscr{G}} \prod_{(x,y) \vdash G} P(Y = y | X \in G).
\end{equation}

\section{Rule Set Learning as Probabilistic Model Selection}
\label{sec:model_selection}

Exploiting the formulation of rule sets as probabilistic models, we define the task of learning a rule set as a probabilistic model selection problem. Specifically, we use the minimum description length (MDL) principle for model selection. 

\subsection{Normalized Maximum Likelihood Distributions for Rule Sets}

The MDL principle is one of the best off-the-shelf model selection methods and has been widely used in machine learning and data mining \cite{grunwald2019minimum}. 
Although rooted in information theory, it has been recently shown that MDL-based model selection can be regarded as an extension of Bayesian model selection \cite{grunwald2019minimum}. 

The core idea of MDL-based model selection is to assign a single probability distribution to the data given a rule set $\ruleset$, the so-called \emph{universal distribution} denoted by $P_{\ruleset}(Y^n|X^n=x^n)$. Informally, $P_{\ruleset}(Y^n|X^n=x^n)$ should be a representative of the rule set model---as a family of probability distributions---$\{P_{\ruleset, \theta}(y^n | x^n)\}_\theta$. The theoretically optimal ``representative" is defined to be the one that has minimax regret, i.e., 
\small{
\begin{equation}
\label{eq:define_regret}
	  \arg \min_{P_{\ruleset}} \max_{z^n \in \mathscr{Y}^n} -\log_2 P_{\ruleset} (Y^n = z^n|X^n = x^n) - \left(-\log_2 P_{\htheta(x^n, z^n)} (Y^n = z^n|X^n = x^n)\right).
\end{equation}
}
We write the parameter estimator as $\htheta(x^n, z^n)$ to emphasize that it depends on the values of the target variables $Y^n$. The unique solution to $P_{\ruleset}$ of Equation~\ref{eq:define_regret} is the so-called normalized maximum likelihood (NML) distribution:
\begin{equation}
    P^{NML}_{\ruleset}(Y^n=y^n|X^n=x^n) = \frac{P_{\ruleset, \htheta(x^n, y^n)}(Y^n=y^n|X^n=x^n)}{\sum_{z^n \in \mathscr{Y}^n} P_{\ruleset, \htheta(x^n, z^n)}(Y^n = z^n|X^n=x^n)}.
\end{equation}
That is, we ``normalize" the distribution $P_{\ruleset, \htheta}(.)$ to make it a proper probability distribution, which requires the sum of all possible values of $Y^n$ to be 1. Hence, we have $\sum_{z^n \in \mathscr{Y}^n} P^{NML}_{\ruleset}(Y^n=z^n|X^n=x^n) = 1$ \cite{grunwald2019minimum}.

\subsection{Approximating the NML Distribution}

A crucial difficulty in using the NML distribution in practice is the computation of the normalizing term $\sum_{z^n} P_{\htheta(x^n, z^n)}(Y^n=z^n|X^n=x^n)$. Efficient algorithms almost only exist for exponential family models \cite{grunwald2019minimum}, hence we approximate the term by the product of the normalizing terms for the individual rules. 

\smallskip
\noindent \textbf{NML distribution for a single rule.}
For an individual rule $S \in \ruleset$, we write all instances covered by $S$ as $(x^S, y^S)$, in which $y^S$ can be regarded as a realization of the random vector $Y^S = (Y, ..., Y)$, and $Y^S$ takes values in $\mathscr{Y}^{|S|}$, the $|S|$-ary Cartesian power of $\mathscr{Y}$. Then, the NML distribution for $P_S(Y)$ equals
\begin{equation}
    P^{NML}_S(Y^S = y^{S}|X^S = x^S) = \frac{\hat{P}_S(Y^S = y^S|X^S = x^S)}{\sum_{z^S \in \mathscr{Y}^S} \hat{P}_S(Y^S = z^S|X^S = x^S)}.
\end{equation}
Note that $\hat{P}_{S}$ depends on the values of $z^S$. As $\hat{P}_S(Y)$ is a categorical distribution, the normalizing term can be written as $\mathcal{R}(|S|, |\mathscr{Y}|)$, a function of $|S|$---the rule's coverage---and $|\mathscr{Y}|$---the number of unique values that $Y$ can take~\cite{mononen:08:sub-lin-stoch-comp}:
\begin{equation}
    \mathcal{R}(|S|, |\mathscr{Y}|) = \sum_{z^S \in \mathscr{Y}^S} \hat{P}_S(Y^S = z^S|X^S = x^S),
\end{equation}
which can be efficiently calculated in sub-linear time~\cite{mononen:08:sub-lin-stoch-comp}.

\smallskip
\noindent \textbf{The approximate NML distribution.}
We propose to approximate the normalizing term of $P^{NML}_{\ruleset}$ as the product of the normalizing terms of $P^{NML}_S$ for all $S \in \ruleset$, and propose the approximate-NML distribution as our model selection criterion:
\begin{equation}
    P^{apprNML}_{\ruleset}(Y^n =y^n | X^n=x^n) = \frac{P_{\ruleset, \htheta(x^n, y^n)}(Y^n=y^n|X^n=x^n)}{\prod_{S \in \ruleset} \mathcal{R}(|S|, |\mathscr{Y}|)}.
\end{equation}
Note that the sum over all $S \in \ruleset$ \emph{does} include the ``else rule" $S_0$. Finally, we can formally define the optimal rule set $\ruleset^*$ as
\begin{equation}
    \ruleset^* = \arg \max_{\ruleset} P^{apprNML}_{\ruleset}(Y^n =y^n | X^n=x^n).
\end{equation}
The rationale of using the approximate-NML distribution is as follows. First, it is equal to the NML distribution for a rule set without any overlap, as follows.
\begin{proposition}
Given a rule set $\ruleset$ in which for any $S_i, S_j \in \ruleset$, $S_i \cap S_j = \emptyset$, then $P^{NML}_{\ruleset}(Y^n=y^n|X^n=x^n) = P^{apprNML}_{\ruleset}(Y^n=y^n|X^n=x^n)$.
\end{proposition}

\noindent Second, when overlaps exist in $\ruleset$, approximate-NML puts a small extra penalty on overlaps, which is desirable to trade-off overlap with goodness-of-fit: when we sum over all instances in each rule $S \in \ruleset$, the instances in overlaps are ``repeatedly counted". Third, approximate-NML behaves like the Bayesian information criterion (BIC) asymptotically, which follows from the next proposition.
\begin{proposition}
Assume $\ruleset$ contains $K$ rules in total, including the else rule, and we have $n$ instances. Then $\log \left(\prod_{S \in \ruleset} \mathcal{R}(|S|, |\mathscr{Y}|)\right) = \frac{K(|\mathscr{Y}| - 1)}{2} \log n + \mathcal{O}(1)$, where $\mathcal{O}(1)$ is bounded by a constant w.r.t.\ to $n$.
\end{proposition}
We defer the proofs of the two propositions to the Supplementary Material. 
\section{Learning Truly Unordered Rule Sets from Data}
As our MDL-based model selection criterion unfortunately does not enable efficient search for the optimal model, we resort to heuristics. We first address the challenge of evaluating incomplete rule sets, after which we explain how to grow individual rules in two phases and implement this with beam search. Finally, we show how everything comes together to iteratively learn rule sets from data.

\subsection{Evaluating Incomplete Rule Sets with a Surrogate Score}

When iteratively searching for the next ``best" rule, defining ``best'' is far from trivial: rule coverage and precision are contradicting factors and typical scores therefore combine those two factors in some---more or less---arbitrary way. 

This issue is further aggravated by the iterative rule learning process, in which the intermediate rule set is evaluated as an \emph{incomplete rule set} in each step. Evaluating incomplete rule sets is a challenging task~\cite{furnkranz2005roc}, mainly because any good score needs to simultaneously consider two aspects: 1) how well do all the rules currently in the rule set describe the already covered instances; and 2) what is the ``potential" for the uncovered instances, in the sense that how well can those uncovered instances be described by rules that might be added later?

Without knowing the rules that will be added later, we cannot compute the NML-based criterion for the complete rule set. 
Yet, we should take into account the potential of the uncovered instances. We propose to approximate the latter using a \emph{surrogate score}, which we obtain by fitting a decision tree on the uncovered instances and using the leafs of the resulting tree as a surrogate for ``future" rules. Formally, we define the tree-based surrogate score as
\begin{equation} \label{eq:tree_score}
    L_{\mathcal{T}} (\ruleset) = P^{apprNML}_{\ruleset \oplus \mathcal{T}} (Y^n = y^n | X^n = x^n),
\end{equation}
where $\ruleset \oplus \mathcal{T}$ denotes the surrogate rule set obtained by converting the branches of $\mathcal{T}$ to rules and appending those to $\ruleset$ (parameters are estimated as usual). 

Although the branches of the decision tree learned from the currently uncovered instances may be different from the rules that will later be added to the rule set, using the tree-based surrogate score will make it easier to gradually grow good rule sets. We use decision trees because they are quick to learn and use, and the correspondence of branches to rules makes using them straightforward. We will empirically study the effects of the surrogate score on the predictive performance of rule sets in Section~\ref{sec:exp}.

\subsection{Two-phase Rule Growth}
To avoid having to traverse all possible rules when searching for the rule to add to an incomplete rule set, we resort to a common heuristic: we start with an empty rule and gradually refine it by adding literals---also referred to as \emph{growing} a rule \cite{furnkranz2012foundations}. In contrast to existing methods, we propose a two-phase method.
\begin{figure}[ht]
    \centering
    \includegraphics[width=0.4\textwidth, height=0.3\textwidth]{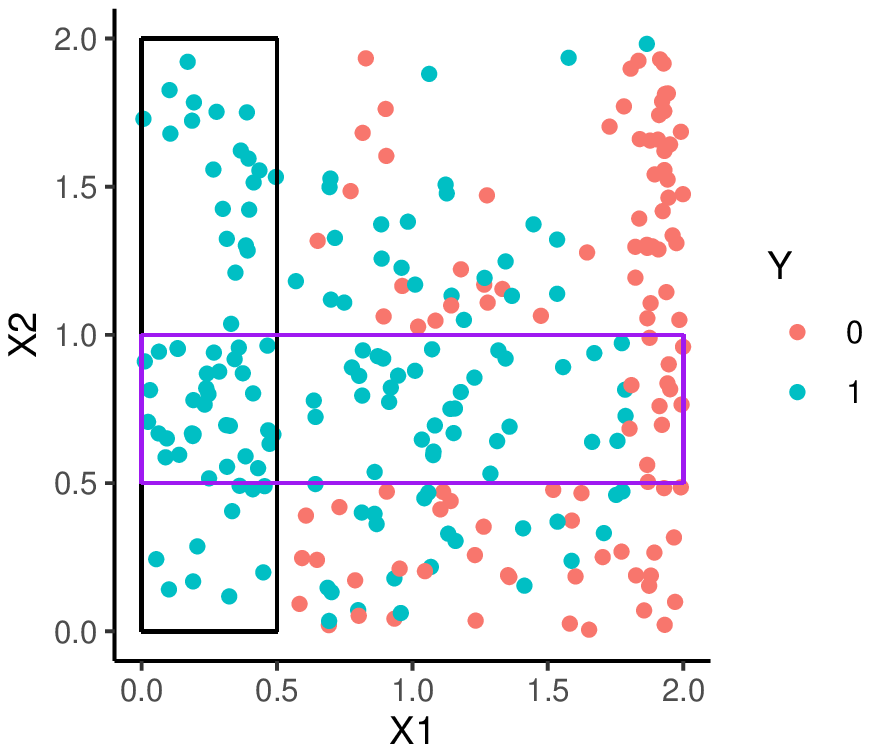}
    \includegraphics[width=0.4\textwidth, height=0.3\textwidth]{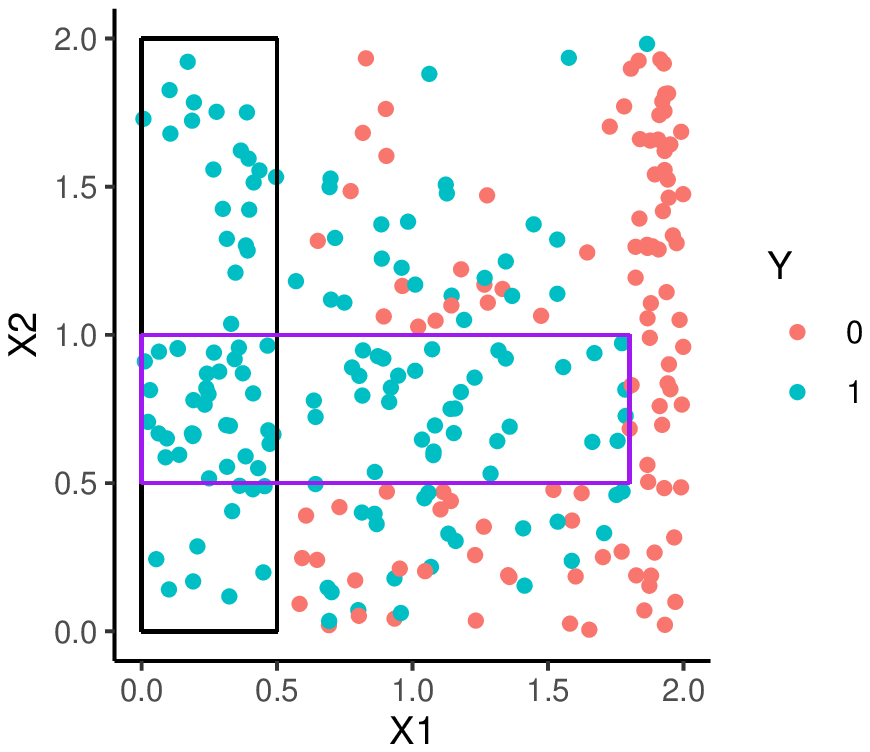}
    \caption{(Left) Simulated data with two overlapping rules: $S_1: X_1 < 0.5$ (outlined in black) and $S_2: 0.5 < X_2 < 1$ (purple). (Right) $S_2$ has  grown to $0.5 < X_2 < 1 \land X_1 < 1.8$, which changes $P(Y|X \in S_2)$ and  resolves the problematic overlap.
    }
    \label{fig:alg1}
\end{figure}
\begin{figure}[ht]
    \centering
    \includegraphics[width=0.4\textwidth, height=0.3\textwidth]{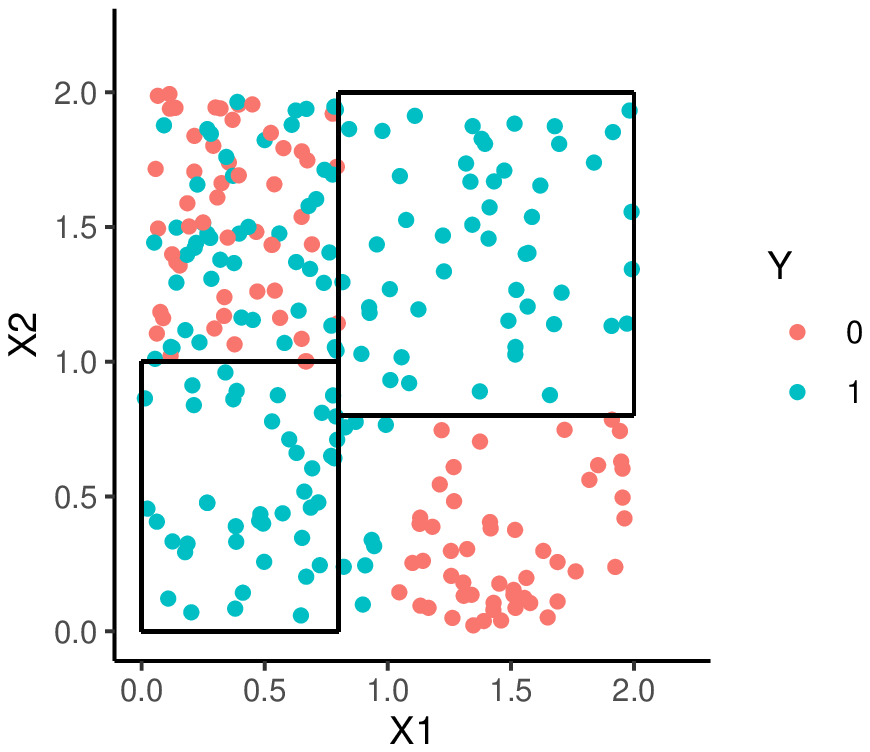}
    \includegraphics[width=0.4\textwidth, height=0.3\textwidth]{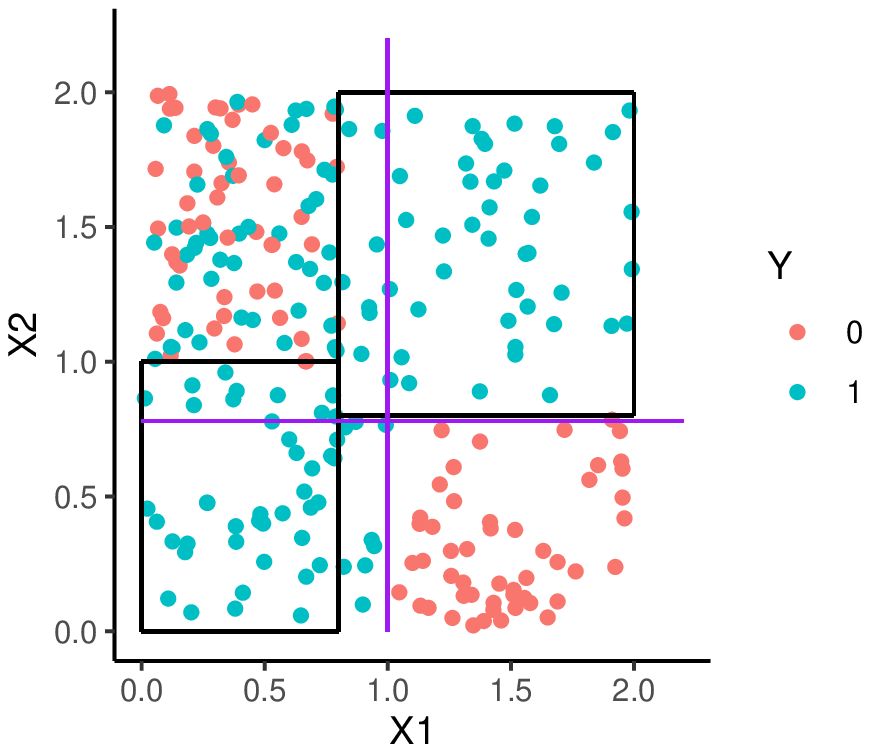}
    \caption{(Left) Simulated data with a rule set containing two rules (black outlines). (Right) Growing a rule to describe the bottom-right instances will create conflicts with existing rules. I.e., adding either $X_1 > 1$ (vertical purple line) or $X_2 < 0.8$ (horizontal purple line) would create a huge overlap that deteriorates the surrogate score (Eq.~\ref{eq:tree_score}).}
    \label{fig:alg2}
\end{figure}
\begin{algorithm}[h] 
\DontPrintSemicolon
  \KwInput{rule set $M$, data $(x^n, y^n)$}
  \KwOutput{A beam that contains the $w$ best rules}
  
  RULE $\assign$ $\emptyset$; Beam $\assign$ [RULE] \tcp*{Initialize the empty rule and beam}
  
  
  BeamList $\assign$ Beam \tcp*{Record all the beams in the beam search}
  
\While{$\operatorname{length}(Beam)  \neq 0$}
  {
    candidates $\assign$ [ ] \tcp*{initialized to store all possible refinements}

    \For {RULE $\in$ Beam}     
    {
        Rs $\assign$ [Append L to RULE for L $\in$ all possible literals] 
        
        candidates.extend(Rs)
    }
    
    Beam $\assign$ the $w$ rules in candidates that have 1) the highest positive $g_{unc}()$, and 2) coverage diversity $> \alpha$ \tcp*{$w$ is the beam width}
    
    \If {$\operatorname{length}$(Beam) $ \neq 0$}
    {    
        BeamList.extend(Beam) \tcp*{extend the BeamList as an array}
    }
    
  }
  
  \For{Rule $\in$ BeamList}
  {
    Beam $\assign$ $w$ rules in BeamList with best $L_{\mathcal{T}}(\ruleset \oplus S_{unc})$ 
  }
  \Return Beam
\caption{Find Next Rule Ignoring Overlaps}
\label{alg:find_next_rule_excl_overlap}
\end{algorithm}	

\noindent \textbf{Motivation.} A rule can only improve the surrogate score---and thus be added to the rule set---if it achieves two goals: 1) it should improve the likelihood of currently uncovered instances (penalized by the approximate-NML normalizing term); and 2) it should not deteriorate the goodness-of-fit of the rule set by creating ``bad'' overlaps. These goals can be conflicting though, for two reasons.

First, it is not necessarily bad to have overlaps between a rule being grown and the current rule set, because the rule and its probability estimates for the target variable may still change. For example, consider the left plot of Figure~\ref{fig:alg1}. If the current rule set consists of $S_1$ (indicated in black), then adding $S_2$ (in purple) would be problematic: this would strongly deteriorate the likelihood of the instances covered by both rules. However, as we further grow $S_2$, as shown in the right plot, we get $P(Y|S_1) = P(Y|S_2)$ and the problem is solved.

Second, rules already in the rule set may become obstacles to growing a new rule. For example, consider the data and rule set with two rules (in black) in Figure~\ref{fig:alg2}. If we want to grow a rule that covers the bottom-right instances, the existing rules form a blockade: the right plot shows how adding either $X_1 > 1$ or $X_2 < 0.8$ to the empty rule (in purple) would create a large overlap with the existing rules, with significantly different probability estimates. 

Therefore, instead of navigating towards the two goals simultaneously, we propose to grow the next rule in two phases: 1) grow the rule as if the instances covered by the (incomplete) rule set are excluded; 2) further grow the rule to eliminate potentially ``bad" overlaps, to further optimize the tree-based score. 

\noindent \textbf{Method.} Given a rule $S$, define $S_{unc}$ as its uncovered ``counterpart", which covers all instances in $S$ not covered by $\ruleset$, i.e., $S_{unc} = S \setminus \cup \{S_i \in \ruleset\}$. Then, given $\ruleset$, the search for the next best rule that optimizes the surrogate tree-based score is divided into two phases. First, we aim to find the $m$ rules for which the uncovered counterparts have the highest surrogate scores, defined as 
\begin{equation}
    L_{\mathcal{T}}(\ruleset \oplus S_{unc}) =  P^{apprNML}_{\ruleset \oplus S_{unc} \oplus \mathcal{T}} (Y^n = y^n | X^n = x^n),
\end{equation}
where $\ruleset \oplus S_{unc} \oplus \mathcal{T}$ denotes $\ruleset$ appended with $S_{unc}$ and all branches of $\mathcal{T}$. Here, $m$ is a user-specified hyperparameter that controls the number of candidate rules that are selected for further refinement in the second phase. In the second phase, we further grow each of these $m$ rules to search for the best one rule that optimizes 
\begin{equation}
        L_{\mathcal{T}}(\ruleset \oplus S) =  P^{apprNML}_{\ruleset \oplus S \oplus \mathcal{T}} (Y^n = y^n | X^n = x^n).
\end{equation}

Given a rule $S$ and its counterpart $S_{unc}$, the score of $S_{unc}$ is an upper-bound on the score of $S$: if $S$ can be further refined to cover exactly what $S_{unc}$ covers, we can obtain $L_{\mathcal{T}}(\ruleset \oplus S_{unc}) = L_{\mathcal{T}}(\ruleset \oplus S_{unc})$. This is often not possible in practice though, and we therefore generate $m$ candidates in the first phase (instead of $1$).

\subsection{Beam Search for Two-phase Rule Growth}

In both phases we aim for growing a rule that optimizes the tree-based score (Equation~\ref{eq:tree_score}); the difference is that we ignore the already covered instances in the first phase. To avoid growing rules too greedily, i.e., adding literals that quickly reduce the coverage of the rule, we use a heuristic that is based on the NML distribution of a single rule and motivated by Foil's information gain \cite{cohen1995ripper}.

\smallskip
\noindent \textbf{Phase 1: rule growth ignoring covered instances.}
We propose the NML-gain to optimize $L_T(\ruleset \oplus S_{unc})$: given two rules $S$ and $Q$, where we obtain $S$ by adding one literal to $Q$, we define the NML-gain as $g_{unc}(S, Q)$:
\begin{equation}
    g_{unc}(S, Q) = \left(\frac{P^{NML}_{{S_{unc}}}( y^{S_{unc}}|x^{S_{unc}})}{|S_{unc}|} - \frac{P^{NML}_{Q_{unc}}( y^{Q_{unc}}|x^{Q_{unc}})}{|Q_{unc}|}\right) |S_{unc}|
\end{equation}
\begin{equation}
    = \left(\frac{\hat{P}_{S_{unc}}( y^{S_{unc}}|x^{S_{unc}})}{\mathcal{R}(|S_{unc}|, |\mathscr{Y}|)\,\,\, |S_{unc}|} - \frac{\hat{P}_{Q_{unc}} ( y^{Q_{unc}}|x^{Q_{unc}})}{\mathcal{R}(|Q_{unc}|, |\mathscr{Y}|) \,\,\, |Q_{unc}|}\right) |S_{unc}|,
\end{equation}

\noindent which we use as the navigation heuristic.

The advantage of having a tree-based score to evaluate rules, besides the navigation heuristic (local score), is that we can adopt beam search,
as outlined in Algorithm~\ref{alg:find_next_rule_excl_overlap}. We start by initializing 1) the rule as an \emph{empty rule} (a rule without any condition), 2) the Beam containing that empty rule, and 3) the BeamRecord to record the rules in the beam search process (Line 1-2). Then, for each rule in the beam, we generate refined candidate rules by adding one literal to it (Ln 5-7). Among all candidates, we select at most $w$ rules with the highest NML-based gain $g_{unc}$, satisfying two constraints: 1) $g_{unc} > 0$; and 2) for each pair of these (at most) $w$ rules, e.g., $S$ and $Q$, their ``coverage diversity" $\frac{|S_{unc} \cap Q_{unc}|}{|S_{unc} \cup Q_{unc}|} > \alpha$, where $\alpha$ is a user-specified parameter that controls the diversity of the beam search \cite{matthijs2012diverse}. We update the Beam with these (at most) $w$ rules (Ln 8-10). We repeat the process until we can no longer grow any rule with positive $g_{unc}$ based on all rules in Beam (Ln 3). Last, among the record of all Beams we obtained during the process, we return the best $w$ rules with the highest tree-based score $L(S_{unc} \cup \ruleset)$ (Ln 11-13). 
\begin{algorithm}[bt]
\DontPrintSemicolon
  \KwInput{training data $(x^n, y^n)$}
  \KwOutput{rule set $M$}
  
  $M \assign \emptyset$; $M\_record$ $\assign [M]$ 
  
  scores $\assign$ [$P_M^{apprNML}(y^n|x^n)$] \tcp*{Record $P_M^{apprNML}$ while growing}
  
  \While{True}
  {
  	$S^*$ $\assign$ FindNextRule$\left(M, (x^n, y^n)\right)$ \tcp*{find the next best rule $S^*$}
  	
  	\If{$S^* = \emptyset$ or $L_{\mathcal{T}}(\ruleset \oplus S) = P_{M \oplus S^*}^{apprNML}(y^n|x^n)$}
  	{
  	    \textbf{Break}
  	}
  	\Else
  	{
  	    $M \assign M \oplus S^*$; $M\_record.append(M)$ \tcp*{update and record $M$}
  	    
  	    scores.append($P_M^{apprNML}(y^n|x^n)$)
  	}
  }
  
  
  \Return the rule set with the maximum score in $M\_record$
    
  	        
  	     
  	         
\caption{Find Rule Set}
\label{alg:ruleset}
\end{algorithm}

\smallskip
\noindent \textbf{Phase 2: rule growth including covered instances.}
We now optimize $L(\ruleset \oplus S)$ and select a rule based on the candidates obtained in the previous step. We first define a navigation heuristic: given two rules $S$ and $Q$, where $S$ is obtained by adding one literal to $Q$, we define the NML-gain $g(S, Q)$ as
\begin{equation}
    g(S, Q) = \left(\frac{\hat{P}_{S}(y^{S_{unc}} | x^{S_{unc}})}{\mathcal{R}(|S_{unc}|, |\mathscr{Y}|)\,\,\, |S_{unc}|} - \frac{\hat{P}_{Q}( y^{Q_{unc}} |  x^{Q_{unc}})}{\mathcal{R}(|Q_{unc}|, |\mathscr{Y}|)\,\,\, |Q_{unc}|}\right) |S_{unc}|.
\end{equation}

Note that the difference between $g(S, Q)$ and $g_{unc}(S, Q)$ is that they use a different maximum likelihood estimator: $\hat{P}_Q$ is the ML estimator based on all instances in $Q$, while $\hat{P}_{Q_{unc}}$ is based on all instances in $Q_{unc}$. 

The algorithm is almost identical to Algorithm~\ref{alg:find_next_rule_excl_overlap}, with four small modifications: 1) the navigation heuristic is replaced by $g(S, Q)$; 2) $L_{\mathcal{T}}(\ruleset \oplus S)$ is used to select the best rule from the BeamRecord instead of $L_{\mathcal{T}}(\ruleset \oplus S_{unc})$ ; and 3) the coverage diversity is based on the rules itself instead of the counterparts; 4) only the best rule is returned.
\subsection{Iterative search for the rule set}
Algorithm~\ref{alg:ruleset} outlines the proposed rule set learner. We start with an empty rule set (Ln 1-2), then iteratively add the next best rule (Ln 3--9) until the stopping criterion is met (Ln 5--6). That is, it stops when 1) the surrogate score equals the `real' model selection criterion (i.e., the model's NML distribution), or 2) no more rules with positive NML-gain can be found. We record the `real' criterion when adding each rule to the set, and pick the one maximizing it (Ln 10). 

\section{Experiments} \label{sec:exp}
We demonstrate that \turs learns rule sets with competitive predictive performance, and that using the surrogate score substantially improves the AUC scores. Further, we demonstrate that \turs achieves model complexities comparable to other rule set methods for multi-class targets.

We here discuss the most important parts of the experiment setup; for completeness, additional information can be found in the Supplementary Material\footnote{The source code is available at \url{https://github.com/ylincen/TURS}}.

\smallskip
\noindent \textbf{Decision trees for surrogate score.} We use CART~\cite{breiman1984classification} to learn the trees for the surrogate score. For efficiency and robustness, we do not use any post-pruning for the decision trees but only set a minimum sample size for the leafs.

\smallskip
\noindent \textbf{Beam width and coverage diversity.} 
We set the coverage diversity $\alpha=0.05$, and beam width $w = 5$. With the coverage diversity as a constraint, we found that $w \in \{5, 10, 20\}$ all give similar results. Due to the limited space, we leave a formal sensitivity analysis of $\alpha$ as future work. 

\smallskip
\noindent \textbf{Benchmark datasets and competitor algorithms.} We test on $13$ UCI benchmark datasets (shown in Table~1), and compare against the following methods: 1) unordered CN2~\cite{clark1991cn2Improve}, the one-versus-rest rule sets method without implicit order among rules; 2) DRS~\cite{zhang2020diverseRuleSets}, a representative multi-class rule set learning method; 3) BRS~\cite{wang2017bayesian}, the Bayesian rule set method for binary classification; 4) RIPPER~\cite{cohen1995ripper}, the widely used one-versus-rest method with orders among class labels; 5) CLASSY~\cite{proencca2020interpretable}, the probabilistic rule list methods using MDL-based model selection; and 6) CART~\cite{breiman1984classification}, the well-known decision tree method, \emph{with} post-pruning by cross-validation. 

\begin{table}[ht]
\centering
\begin{tabular}{lcccc|ccc||c}
  \hline
data & TURS & CN2 & DRS & BRS & CLASSY & RIPPER & CART & TURS \%overlap \\ 
  \hline
anuran & 0.998 & 1.000 & 0.858 & --- & 0.983 & 0.999 & 0.996 & 0.395 \\ 
  avila & 0.968 & 0.978 & 0.530 & --- & 0.954 & 0.997 & 0.988 & 0.286 \\ 
  backnote & \textbf{0.991} & 0.969 & 0.945 & 0.957 & 0.987 & 0.979 & 0.984 & 0.297 \\ 
  car & \textbf{0.978} & 0.633 & 0.924 & --- & 0.945 & 0.980 & 0.971 & 0.063 \\ 
  chess & \textbf{0.995} & 0.536 & 0.823 & 0.945 & 0.991 & 0.995 & 0.994 & 0.264 \\ 
  contracept & \textbf{0.667} & 0.597 & 0.544 & --- & 0.630 & 0.626 & 0.600 & 0.074 \\ 
  diabetes & \textbf{0.766} & 0.677 & 0.628 & 0.683 & 0.761 & 0.735 & 0.661 & 0.155 \\ 
  ionosphere & 0.914 & 0.912 & 0.663 & 0.837 & 0.909 & 0.901 & 0.845 & 0.310 \\ 
  iris & 0.964 & 0.985 & 0.935 & --- & 0.960 & 0.973 & 0.965 & 0.018 \\ 
  magic & \textbf{0.886} & 0.590 & 0.695 & 0.794 & 0.895 & 0.818 & 0.800 & 0.500 \\ 
  tic-tac-toe & 0.972 & 0.826 & 0.971 & 0.976 & 0.983 & 0.954 & 0.847 & 0.231 \\ 
  waveform & \textbf{0.902} & 0.775 & 0.588 & --- & 0.833 & 0.884 & 0.803 & 0.528 \\ 
  wine & 0.954 & 0.962 & 0.810 & --- & 0.961 & 0.945 & 0.932 & 0.031 \\ 
  \hline
  Avg Rank & 2.231 & 4.077 & 5.846 & 5.462 & 3.154 & 3.000 & 4.231 & /\\ 
   \hline
\end{tabular}
\caption{ROC-AUC scores, averaged over 10 cross-validated folds. The rank (smaller means better) is further averaged over all datasets. Among the four \emph{rule set} methods, \turs is substantially better on 7 out 13 datasets (AUC scores in bold).}
\end{table}

\subsection{Results}
\begin{figure}[t]
    \centering
    \includegraphics[height=0.35\textwidth]{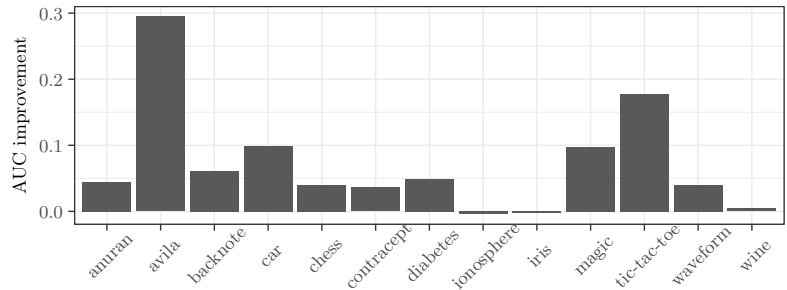}
    \caption{Improvement in AUC by enabling the surrogate score for \textsc{Turs}.}
    \label{fig:surrogate}
\end{figure}
\noindent \textbf{Predictive performance.}
We report the ROC-AUC scores in Table~1. For multi-class classification, we report the weighted one-versus-rest AUC scores, as was also used for evaluating the recently proposed CLASSY method~\cite{proencca2020interpretable}. 

Compared to non-probabilistic rule set methods---i.e., CN2, DRS, and BRS (only for binary targets)---\turs is much better in terms of the mean rank of its AUC scores. Specifically, it performs substantially better on about half of the datasets (shown in bold). 
Besides, it is ranked better than rule list methods, which produce explicitly ordered rules that may be difficult for domain experts to comprehend and digest in practice. 
Next, CART attains AUCs generally inferior to \textsc{Turs}, although it helps \turs to get a higher AUC as part of the surrogate score.

Last, we report the percentage of instances covered by more than one rule for \textsc{Turs} in Table~1, and we show that overlaps are common in the rule sets obtained for different datasets. This empirically confirms that our way of formalizing rule sets as probabilistic models, i.e., treating overlaps as uncertainty and exception, can indeed lead to improved predictive performance, as the overlapping rules are a non-negligible part of the model learned from data and hence indeed play a role. 



\noindent \textbf{Effects of the surrogate score.}
Figure~\ref{fig:surrogate} shows the difference in AUC obtained by our method with and without using the surrogate score (i.e., without surrogate score means replacing it with the final model selection criterion). We conclude that the surrogate score has a substantial effect on learning better rule sets, except for three ``simple" datasets, of which the sample sizes and the number of variables are small, as shown in Table~2 (Left).
\begin{table}[t]
\centering
\begin{tabular}{ll||lrrrr|rrr}
  \hline
\#instances & \#features &data & TURS & CN2 & DRS & BRS & CLASSY & RIPPER & CART \\ 
  \hline
1372 &   5&backnote & 42 & 41 & 55 & 22 & 22 & 16 & 94 \\ 
1473 &  10 &  contracept & 75 & 275 & 73 & --- & 14 & 14 & 6241 \\ 
 768 &   9 & diabetes & 55 & 152 & 131 & 10 & 10 & 6 & 827 \\ 
 150 &   5 & iris & 7 & 9 & 23 & --- & 3 & 3 & 9 \\ 
958 &  10 &  tic-tac-toe & 86 & 90 & 108 & 24 & 27 & 60 & 816 \\ 
178 &  14 &  wine & 10 & 6 & 134 & --- & 6 & 5 & 15 \\ 
1728 &   7 &   car & 211 & 163 & 325 & --- & 92 & 111 & 718 \\ 
7195 &  24 &  anuran & 74 & 37 & 407 & --- & 49 & 7 & 96 \\ 
3196 &  37 &  chess & 299 & 316 & 482 & 21 & 37 & 44 & 355 \\ 
  351 &  35 & ionosphere & 50 & 30 & 261 & 14 & 6 & 5 & 101 \\ 
  5000 &  22 & waveform & 707 & 802 & 60 & --- & 139 & 115 & 3928 \\ 
 20867 &  11 &     avila & 890 & 1296 & 179 & --- & 988 & 574 & 8145 \\ 
 19020 &  11 & magic & 1321 & 2238 & 48 & 23 & 256 & 69 & 22566 \\ 
  \hline
  &&Avg Rank &2.15 & 2.46 &2.77& 1.00& ---&---&--- \\
\hline
\end{tabular}
\caption{Left: The sample sizes and number of features of datasets. Right: total number of literals, i.e., average rule lengths $\times$ number of rules in the set, averaged over 10-fold cross-validation. The rank is averaged over all datasets, for rule sets methods only.}
\end{table}

\noindent \textbf{Model complexity.} 
Finally, we compare the `model complexity' of the rule sets for all methods. As this is hard to quantify in a unified manner, as a proxy we report the \emph{total number of literals in all rules in a rule set}, averaged over 10-fold cross-validation (the same as used for the results reported in Table~1).

We show that among all rule set methods (TURS, CN2, DRS, BRS), \textsc{Turs} has better average ranks than both CN2 and DRS. Although BRS learns very small rule sets, it is only applicable to binary targets and its low model complexity also brings worse AUC scores than \textsc{Turs}. Further, although rule list methods (CLASSY, RIPPER) generally have fewer literals than rule sets methods,  this does not make rule lists easy to interpret, as every rule depends on all previous rules. Last, we empirically confirm that tree-based method CART produces much larger rule sets.


\section{Conclusion}
We formalized the problem of learning truly unordered probabilistic rule sets as a model selection task. We also proposed a novel, tree-based surrogate score for evaluating incomplete rule sets. Building upon this, we developed a two-phase heuristic algorithm that learns rule set models that were empirically shown to be accurate in comparison to competing methods. 

For future work, we will study the practical use of our method with a case study in the health care domain. This involves investigating how well our method scales to larger datasets. Furthermore, a user study will be performed to investigate whether, and in what degree, the domain experts find the truly unordered property of rule sets obtained by our method helps them comprehend the rules better in practice, in comparison to rule lists/sets with explicit or implicit orders.

\smallskip
\small{
\noindent \textbf{Acknowledgements.} We are grateful for the very inspiring feedback from the anonymous reviewers. This work is part of the research programme `Human-Guided Data Science by Interactive Model Selection' with project number 612.001.804, which is (partly) financed by the Dutch Research Council (NWO).
}

%
%
%
\bibliographystyle{splncs04}
\bibliography{main}

\end{document}